\pdfoutput=1
\documentclass[11pt]{article}
\usepackage[a4paper,margin=1in]{geometry}
\usepackage[T1]{fontenc}
\usepackage{amsmath,amsthm,mathtools}
\usepackage{newtxtext,newtxmath}
\usepackage{graphicx,booktabs,multirow,xcolor}
\usepackage[numbers,sort&compress]{natbib}
\usepackage[colorlinks,linkcolor=black,citecolor=black,urlcolor=black]{hyperref}
\usepackage{enumitem}
\setlist{nosep,leftmargin=*}

\newtheorem{proposition}{Proposition}
\newtheorem{theorem}[proposition]{Theorem}
\newtheorem{corollary}[proposition]{Corollary}
\newtheorem{lemma}[proposition]{Lemma}
\theoremstyle{definition}
\newtheorem{example}[proposition]{Example}
\theoremstyle{remark}

\newcounter{algo}
\newcommand{\Fbar}{\bar F}
\newcommand{\Gbar}{\bar G}
\newcommand{\Hit}{\mathrm{Hit}}
\newcommand{\Prob}{\mathbb{P}}
\newcommand{\E}{\mathbb{E}}

\title{The Price of Scale in Proactive Memory Recall}
\author{Yichuan Chen\thanks{Corresponding author: \texttt{smallboychen@gmail.com}.}
\and Xin Wang\\[2pt]
\small School of Art and Design, Zhejiang Sci-Tech University, Hangzhou 310018, China}
\date{}

\begin{document}
\maketitle

\begin{abstract}
A long-lived assistant that brings up memories without being asked must decide at every turn whether anything in its store is worth an interruption. We show that this decision is a multiple-testing problem whose difficulty grows with the store size $N$. On streaming replays of LongMemEval, a similarity threshold tuned to $5\%$ false alarms at $N=8$ reaches $84$--$100\%$ at $N=2048$, across embedding models and for explicit and implicit cues, and LoCoMo, whose stores reach $N=256$, shows the same growth; LLM judges show it from a lower base. We prove that a gate which thresholds the best of $N$ independent irrelevant scores at false-alarm rate $\alpha$ operates at a per-item false-positive rate of about $-\ln(1-\alpha)/N$, which pins its hit rate to the far left of the single-comparison ROC curve and, under regular variation, makes it regularly varying in $N$ with index $-\kappa$. If the relevant memory is equally likely to be any of the $N$, then whatever a gate reads and however it decides, at false-alarm rate $\alpha$ it hits at most $\alpha+\sqrt{\alpha(1-\alpha)\chi^2/N}$, where $\chi^2$ is the chi-square divergence between one relevant and one irrelevant memory, and with binormal scores of equal variance thresholding the best score is the optimal gate. Under a global threshold, replayed hit rates stay within the bounds this gives for heterogeneous cues up to $0.3$ points, and hit rates carry over from half of the questions to the other half to within $6.0$ points. At exactly matched false alarms, with explicit questions and a small embedding model, the deployed score, which adds per-cue normalisation and a normalised lexical channel, recalls $14$--$26$ points more than raw cosine. A cue-swap conformal gate, which ranks the current cue among other cues on the same store, has a finite-sample false-alarm guarantee under exchangeability. With a background drawn from the user's unfiltered history it keeps false alarms at $3$--$6\%$ on LongMemEval at the hit rate of a label-tuned threshold; on LoCoMo it overshoots or turns conservative depending on where its reference cues come from. Used as a drop-in gate in a chronological replay, with the user's own earlier messages as references, it fires on $5$--$8\%$ of explicit questions on stores without their answer and speaks on $6$--$8\%$ of session-opening turns once the store holds about $20$ turns, where at $N=2048$ the deployed gate speaks on $57\%$ and fixed thresholds on $78$--$87\%$, at a median of $0.4$\,s of CPU time per decision on one thread. A cross-encoder reranker pays the same price of scale, and on a prospective-memory benchmark whose triggers are harmless on their own, similarity gates recall at most $9\%$ of constraints at controlled false alarms. One human and two LLM annotators judge most false alarms unhelpful for a response. Code and data: \url{https://github.com/RiverChenStillCompiling/price-of-scale}.
\end{abstract}

\section{Introduction}

Human recall follows the situation. After a long run a cold drink comes to mind, while a bath taken ten years ago stays dormant, although both are stored. Assistants that live with a user for months are now expected to do something similar. While helping with today's task, they should notice that something said weeks ago matters and say so without being asked~\citep{wu2026remember,zeng2026memoproagent,zhang2026triggerbench,liu2026pmbench}. In question answering the query says what to retrieve. Here the assistant has to decide on its own whether to speak, and every unwarranted reminder costs the user attention.

The common design compares the current situation with every stored memory and speaks if the best match is good enough. This is a multiple comparison, and the number of comparisons grows with the life of the assistant. A threshold that keeps quiet with eight memories fires often with two thousand, because the largest of many irrelevant similarity scores keeps rising. Work on retrieval and sequence search has long recognised that significance depends on database size~\citep{karlin1990methods,reimers2021curse}. We work out what this implies for an assistant that must stay quiet and still remember, and how much it can recall once its false-alarm rate is held fixed. No gate, including a language model that reads every memory, recalls more than a bound set by store size and by how well a single relevant memory can be told from an irrelevant one.

This paper makes four contributions.
\begin{itemize}
\item \textbf{Measurement.} We replay LongMemEval~\citep{wu2025longmemeval} and LoCoMo~\citep{maharana2024locomo} as growing memory stores and measure false alarms and hits of eight gates, a cross-encoder and two LLM judges from $N=8$ to $N=2048$ (\S\ref{sec:fa}), in random and chronological replays and on the prospective-memory benchmark TriggerBench~\citep{zhang2026triggerbench} (\S\ref{sec:temporal}, \S\ref{sec:triggerbench}). Fixed thresholds lose control within a few doublings; per-cue $z$-scores with a Gaussian threshold law slow the growth; LLM judges show the same growth from a lower base.
\item \textbf{Theory.} We prove that a gate which thresholds the best of $N$ independent irrelevant scores at false-alarm rate $\alpha$ works at per-item false-positive rate $p_N=1-(1-\alpha)^{1/N}\approx -\ln(1-\alpha)/N$, give an exact integral for its hit rate, and bound it within a factor $(1-\alpha)^{(N-1)/N}$ of the ROC value $R(p_N)$ (Proposition~\ref{prop:sandwich}). Under regular variation the hit rate decays as $N^{-\kappa}$ up to a slowly varying factor, with an explicit constant (Theorem~\ref{thm:power}); for a binormal model we give two-sided capacity bounds (Corollary~\ref{cor:capacity}); a covariance identity describes heterogeneous cues under a global threshold (Corollary~\ref{cor:mixture}). When the relevant memory is equally likely to be any of the $N$, the hit rate of every gate at false-alarm rate $\alpha$, including judges that read all candidates, is at most $\alpha+\sqrt{\alpha(1-\alpha)\chi^2/N}$, with $\chi^2$ the divergence between one relevant and one irrelevant memory (Theorem~\ref{thm:anygate}); with binormal scores of equal variance, thresholding the best score is the optimal gate. Replayed hit rates stay within the bounds of Corollary~\ref{cor:mixture} up to $0.3$ points, and hit rates carry over to held-out questions to within $6.0$ points (\S\ref{sec:theory-check}).
\item \textbf{Mechanism.} The largest irrelevant score grows much faster in the lexical than in the semantic channel, at the mean and at the $95\%$ quantile (\S\ref{sec:tails}), so one threshold law cannot serve both. At exactly matched false alarms, per-cue normalisation of the semantic score and a normalised lexical channel account for the advantage of the deployed score over raw cosine; its eligibility rules change the hit rate by less than $1$ point (\S\ref{sec:matched}).
\item \textbf{Method and artefact.} A cue-swap conformal gate (Proposition~\ref{prop:swap}, Algorithm~\ref{alg:swap}) holds false alarms near the nominal level on LongMemEval without labels in its threshold or background, at the hit rate of a label-tuned threshold, and in a chronological replay with only the user's earlier messages as references its false-alarm and firing rates stay flat as the store grows; on LoCoMo and TriggerBench, where its reference cues come from another conversation or scenario, it loses that calibration (\S\ref{sec:swapexp}, \S\ref{sec:temporal}). We release the replay code, the annotated data, the implicit cues and a stand-alone implementation of the gates.
\end{itemize}
Finally, one human and two LLM annotators label 240 false alarms (\S\ref{sec:annotation}); by majority vote $3$--$12\%$ of those labelled after a translation fix are useful near misses.

\section{Related work}

\paragraph{Proactive memory.} Recent agents add a component that decides whether to inject a reminder. \citet{wu2026remember} train a memory agent alongside an action agent and show that selective intervention beats always-on injection; MemoProAgent~\citep{zeng2026memoproagent} gates memories from a year-long personal store and finds timing as hard as content; PRISM~\citep{fu2026prism} reports false-alarm rates above $50\%$ for proactive baselines. TriggerBench~\citep{zhang2026triggerbench} and PM-Bench~\citep{liu2026pmbench} evaluate prospective memory and observe a precision--recall trade-off and an ``always remind'' failure mode. MemUse~\citep{sumida2026memuse} finds that a system answering $79\%$ of direct memory questions uses $8\%$ of relevant facts in conversation. Earlier agent memories retrieve by relevance, recency and importance~\citep{park2023generative,packer2023memgpt}, and recent work studies admission into memory~\citep{zhang2026amac}, spreading activation~\citep{jiang2026synapse} and abstention when retrieving for coding agents~\citep{iscan2026learning}. These systems learn or prompt their gates; none formally characterises how a gate's errors scale with store size or provides a finite-sample false-alarm guarantee.

\paragraph{When to retrieve.} Adaptive retrieval decides whether a language model should consult external text, from token confidence~\citep{jiang2023flare}, reflection tokens~\citep{asai2024selfrag}, question complexity~\citep{jeong2024adaptiverag} or entity popularity~\citep{mallen2023trust}; RetrievalQA~\citep{zhang2024retrievalqa} notes that calibration-based variants depend on threshold tuning. These methods decide whether the model needs outside knowledge. Proactive recall asks whether any stored item is relevant at all, a question whose answer depends on the size of the store.

\paragraph{Score distributions and significance.} Information retrieval models non-relevant scores as exponential and relevant ones as normal~\citep{manmatha2001modeling,arampatzis2009stop}; BLAST scales its expected number of chance hits with database size~\citep{karlin1990methods}; dense retrievers return irrelevant passages more often as the index grows~\citep{reimers2021curse}; and some items are nearest neighbours of many queries in high dimension~\citep{radovanovic2010hubs}. Our growth laws for the null maximum are consequences of these tail shapes and classical extreme-value theory~\citep{leadbetter1983extremes}. We add an analysis of the hit rate that remains at a fixed false-alarm rate and of its consequences for agent memory. Our bound for every gate is the second-moment argument of sparse-signal detection, where testing for a single signal among $N$ noise variables is a classical problem~\citep{ingster1997some,donoho2004higher,ariascastro2011global}.

\paragraph{Error control and adaptive thresholds.} Neyman--Pearson classification fixes the type-I error of a classifier and maximises its power~\citep{tong2018np}; we fix the false-alarm rate of a gate whose number of comparisons grows, which is what makes the attainable power shrink. Conformal p-values give finite-sample validity under exchangeability~\citep{vovk2005algorithmic,bates2023outliers,jin2023selection}, conformal risk control bounds the expectation of monotone losses~\citep{angelopoulos2024crc}, weighted variants bound the cost of departures from exchangeability~\citep{barber2023beyond}, and online procedures control error rates over streams of decisions~\citep{javanmard2018online,bao2024cap,lu2026gaif}. In retrieval-augmented generation, conformal methods guarantee that the retrieved context covers the answer~\citep{rouzrokh2024conflare} or bound the risk of the generated text~\citep{kang2024crag}; both concern stores that hold something relevant, while our gate controls firing on stores that do not. Radar detection has long set thresholds from a local reference window to keep a constant false-alarm rate and has studied how contaminated references bias them~\citep{rohling1983cfar}; per-cue normalisation and the cue-swap reference set play these roles here. The cue-swap gate is a direct use of a conformal p-value, with the user's earlier messages as its reference set.

\section{Proactive recall as repeated testing}
\label{sec:setup}

At a turn the assistant sees a cue $c$ (the user's message, possibly with context) and a store of $N$ memories. A scoring rule gives each memory a score $s_i(c)$; a gate speaks if $\max_i s_i(c)$ exceeds a threshold and then brings up the arg-max. If the store contains no memory relevant to $c$, speaking is a \emph{false alarm}; if it contains a relevant memory, the turn is a \emph{hit} when the gate speaks and the arg-max is that memory.

\paragraph{Streaming replay.} We turn memory benchmarks into stores of controlled size. For each question, the \emph{null pool} consists of the user turns from sessions that contain no evidence for it; evidence turns are those labelled as containing the answer. A \emph{null trial} draws $N$ items from the null pool without replacement; a \emph{positive trial} draws $N-1$ such items and adds one evidence turn. The false-alarm rate is the firing rate on null trials and the hit rate is the proportion of positive trials that fire on the evidence turn. Gates that standardise scores also see a \emph{background}: up to $300$ turns of the same user's unfiltered history, evidence sessions included, drawn outside the store after the store is drawn. A background drawn from the labelled null pool is reported in Appendix~\ref{app:background}. Gates that need a threshold are tuned on a hash-defined half of the questions and evaluated on the other half; confidence intervals use a bootstrap over questions, or over conversations for LoCoMo.

\paragraph{Gates.} We compare (details in Appendix~\ref{app:gates}):
(i) a cosine threshold and (ii) a tf-idf threshold, each tuned once to $5\%$ false alarms at $N=8$;
(iii) \emph{recall-v1}, the gate of a deployed memory plugin, which standardises both channels per cue against a background sample and uses the threshold $\max(2,\sqrt{2\ln 2N})+0.5$;
(iv) the same $z$-score with its threshold frozen at $N=8$;
(v) an extreme-value gate that fits an exponential tail to each channel and applies a Bonferroni level $\alpha/2N$;
(vi) oracle versions of the cosine and $z$-score gates, re-tuned with labels at every $N$; and
(vii) the cue-swap conformal gate of \S\ref{sec:swap}.
We also run LLM judges that read the five memories with the highest cosine and either name one to bring up or stay silent.

\section{Theory: the price of scale}
\label{sec:theory}

Fix a cue. Let $X_1,\dots,X_N$ be the scores of irrelevant memories and $Y$ the score of a relevant one. We assume:
\begin{enumerate}[label=(A\arabic*),leftmargin=3em]
\item $X_1,\dots,X_N$ are i.i.d.\ with distribution function $F$, continuous and strictly increasing on its support;
\item $Y$ is independent of the $X_i$ and has a continuous distribution function $G$;
\item the gate speaks when the largest score in the store is at least $t$ and brings up the arg-max.
\end{enumerate}
Write $\Fbar=1-F$, $\Gbar=1-G$, and define the single-comparison ROC curve $R(p)=\Gbar(\Fbar^{-1}(p))$, the probability that the relevant memory clears the threshold at which one irrelevant memory clears it with probability $p$. Let $U_i=\Fbar(X_i)$ and $U=\Fbar(Y)$. Then the $U_i$ are i.i.d.\ uniform on $(0,1)$, independent of $U$, and $\Prob(U\le u)=R(u)$; we extend $R$ to the distribution function of $U$ on $\mathbb R$, which may put mass at $0$ (Lemma~\ref{lem:transform}).

\begin{proposition}[Sandwich]\label{prop:sandwich}
Choose $t_N$ with $F(t_N)^N=1-\alpha$, so that the null false-alarm rate is $\alpha$, and let $p_N=\Fbar(t_N)=1-(1-\alpha)^{1/N}$. A store with $N-1$ irrelevant memories and one relevant memory gives
\[
\Hit_N=\int_{[0,p_N]}(1-u)^{N-1}\,dR(u),\qquad
(1-\alpha)^{\frac{N-1}{N}}R(p_N)\le\Hit_N\le R(p_N),
\]
and $\alpha/N\le p_N\le\lambda/N$ with $\lambda=-\ln(1-\alpha)$.
\end{proposition}

The gate is forced to work at a per-item false-positive rate of about $\lambda/N$, so its hit rate is read off the far left end of the ROC curve, and the end moves further left as the store grows. The proof (Appendix~\ref{app:proofs}) is a direct computation in the uniform scale.

\begin{theorem}[Power-law decay]\label{thm:power}
If $R(p)=p^\kappa\ell(p)$ with $\kappa>0$ and $\ell$ slowly varying at $0$, then
\[
\lim_{N\to\infty}\frac{\Hit_N}{R(p_N)}=C(\alpha,\kappa):=\int_0^1\kappa v^{\kappa-1}(1-\alpha)^v\,dv\in[1-\alpha,1],
\qquad
\Hit_N\sim C(\alpha,\kappa)\,\lambda^\kappa N^{-\kappa}\ell(1/N).
\]
\end{theorem}

\begin{example}[Binormal scores]\label{ex:binormal}
If $X\sim\mathcal N(0,1)$ and $Y\sim\mathcal N(d,s^2)$ with $d,s>0$, then $\kappa=1/s^2$ and $\ln\ell(p)=(d/s^2)\,t+(1/s^2-1)\ln t+c+o(1)$ with $t=\Phi^{-1}(1-p)$; $\ell$ is slowly varying, yet grows faster than any power of $\ln(1/p)$. The local exponent of the ROC curve is $\kappa_{\mathrm{eff}}(p)=d\ln R/d\ln p=h((t-d)/s)/(s\,h(t))$, where $h=\phi/\bar\Phi$, and $\kappa_{\mathrm{eff}}(p)=s^{-2}(1-d/t)+O(t^{-2})$ as $p\downarrow0$ with $d,s$ fixed. It therefore approaches $1/s^2$ from below for small enough $p$. The local exponent of $R$ is not the log--log slope of $\Hit_N$ in $N$.
\end{example}

\begin{corollary}[Capacity]\label{cor:capacity}
In the binormal model let $t_\beta=d-s\Phi^{-1}(\beta)$, $q_\beta=\bar\Phi(t_\beta)$ and $N^*(\beta)=\ln(1-\alpha)/\ln(1-q_\beta)$. Then $\Hit_N<\beta$ for every $N>N^*(\beta)$; if $\beta'=\beta/(1-\alpha)<1$, then $\Hit_N\ge\beta$ for every $N\le N^*(\beta')$. For $t_\beta>0$,
\[
\ln N^*=\ln\lambda+\tfrac12 t_\beta^2+\ln t_\beta+\tfrac12\ln 2\pi+\eta,\qquad \eta\in\big[\ln(1-q_\beta),\,\ln(1+t_\beta^{-2})\big].
\]
\end{corollary}

Writing $N_{\mathrm{cap}}(\beta)$ for the largest $N$ with $\Hit_N\ge\beta$, if there is one, the corollary gives $\lfloor N^*(\beta/(1-\alpha))\rfloor\le N_{\mathrm{cap}}(\beta)\le\lfloor N^*(\beta)\rfloor$; $N^*$ is an exact formula for the upper bound, and the capacity itself lies between the two. As $t_\beta\to\infty$, $N^*(\beta)\sim\lambda\sqrt{2\pi}\,t_\beta e^{t_\beta^2/2}$, so on the leading exponential scale a modest gain in separation multiplies the bound.

Experiments use one threshold for all cues, and cues differ. Let $\alpha_c(\tau)$ be the false-alarm rate of cue $c$ at a global threshold $\tau$, with $\E_c\alpha_c=\alpha$, and $\Gbar_c(\tau)=\Prob(Y_c\ge\tau)$ its pass rate, which we define with $\ge$ because scores on a finite history are discrete.

\begin{corollary}[Heterogeneous cues]\label{cor:mixture}
$\E_c\Hit_c\le\E_c\Gbar_c(\tau)$ and
\[
\E_c\Hit_c\ \ge\ \E_c[(1-\alpha_c)\Gbar_c(\tau)]=(1-\alpha)\,\E_c\Gbar_c(\tau)-\mathrm{Cov}_c\big(\alpha_c,\Gbar_c(\tau)\big).
\]
Both bounds also hold when, for each cue, null and positive stores draw their irrelevant items as simple random samples without replacement from the same pool of $M\ge N$ irrelevant memories, with the evidence item independent of them.
\end{corollary}

Cues whose baseline similarity is high, the hubs of~\citet{radovanovic2010hubs}, can raise both $\alpha_c$ and $\Gbar_c$, which makes the covariance positive, as it is at every $N$ in our data (\S\ref{sec:theory-check}). \S\ref{sec:matched} compares per-cue normalisation with a global threshold.

The exact hit rate on a finite pool replaces $F^{N-1}$ by a hypergeometric factor. For a cue whose pool has $M$ irrelevant memories, $L(y)$ of them scoring strictly below $y$, and no ties between evidence and pool scores,
\[
\Hit_N(\tau)=\int_{[\tau,\infty)}\binom{L(y)}{N-1}\Big/\binom{M}{N-1}\,dG(y).
\]
We use this finite-pool form, averaged over cues, as the exact model in \S\ref{sec:theory-check}.

\subsection{A bound for every gate}
\label{sec:anygate}

The results so far concern gates that threshold the best score. A gate could instead read every memory with a language model, combine scores in another way, or randomise. None of this removes the price. Fix a cue and let $Z_1,\dots,Z_N$ be the memories in the store, random elements of measurable spaces in whatever form the gate reads them. On a null store the $Z_i$ are independent with laws $Q_1,\dots,Q_N$. On a positive store a position $J$ is drawn with $\Prob(J=j)=\pi_j$, memory $J$ has law $P_J$ in place of $Q_J$, and the others keep their null laws, independently. A gate is a measurable rule that maps the store to silence or to one of its positions. It may use randomisation or other inputs, provided their law is the same on null and positive stores and they are independent of the store and of $J$. It hits when it brings up position $J$. Write $\chi^2(P\,\|\,Q)=\int(dP/dQ)^2\,dQ-1$ for the chi-square divergence, set to $\infty$ unless $P$ is absolutely continuous with respect to $Q$, let $S=\{j:\pi_j>0\}$, and write $L_j=dP_j/dQ_j$ when it exists.

\begin{theorem}[Every gate]\label{thm:anygate}
\begin{enumerate}[label=(\roman*),leftmargin=2em]
\item Every gate whose false-alarm rate on null stores is $\alpha$ satisfies
\[
\Hit\ \le\ \Prob(\text{speak}\mid\text{positive store})\ \le\ \alpha+\sqrt{\alpha(1-\alpha)\textstyle\sum_{j\in S}\pi_j^2\,\chi^2(P_j\,\|\,Q_j)}\,,
\]
where the right side is infinite if one of these divergences is. If the position is uniform and $\chi^2(P_j\,\|\,Q_j)\le\chi^2$ for every $j$, then $\Hit\le\alpha+\sqrt{\alpha(1-\alpha)\chi^2/N}$.
\item Let $P_j\ll Q_j$ for every $j\in S$. If $0<\alpha\le1$, the largest hit rate at false-alarm rate $\alpha$ is attained by the gate that speaks when $\max_{j\in S}\pi_jL_j(Z_j)$ exceeds a threshold, randomising at the threshold, and brings up the arg-max. If $\alpha=0$, every gate has hit rate $0$.
\end{enumerate}
\end{theorem}

Applied to the stored records themselves, the divergence is a property of the task and of what the store records, and by the data-processing inequality no model's representation or score of a single memory has a larger one. When relevant memories concentrate on a few positions, such as recent ones, store size enters the bound through the effective number of positions $1/\sum_j\pi_j^2$. With uniform positions and $0<\alpha<1$, hitting at rate $\alpha+\delta$ with $0<\delta\le1-\alpha$ requires $\chi^2\ge N\delta^2/(\alpha(1-\alpha))$, so the separation that a single comparison must achieve grows linearly in $N$ on the chi-square scale. In the binormal model of Example~\ref{ex:binormal}, with the scores as the memories and $s^2<2$,
\[
\chi^2=\frac{e^{d^2/(2-s^2)}}{s\sqrt{2-s^2}}-1,
\]
which is $e^{d^2}-1$ for $s=1$; every gate on these scores then needs $d^2\ge\ln\big(1+N\delta^2/(\alpha(1-\alpha))\big)$, a separation of order $\sqrt{\ln N}$.

Part (ii) is sharper. With $s=1$ and uniform positions, $L_j(z)=e^{dz-d^2/2}$ is increasing in $z$, so the optimal gate is the best-score gate itself, and the hit rate of Proposition~\ref{prop:sandwich} is the largest that any gate on these scores can achieve. That gate needs a separation of about $\sqrt{2\ln N}$ (Corollary~\ref{cor:capacity}). The gap to the $\sqrt{\ln N}$ of part (i) is a loss of the second-moment method. Writing $d=\sqrt{2r\ln N}$ for one relevant item at a uniform position, no test has asymptotic power above its level when $r<1$, and the best-score gate detects consistently when $r>1$~\citep{ingster1997some,ariascastro2011global}. For detecting a relevant memory rather than bringing it up, the most powerful test is the likelihood ratio $N^{-1}\sum_ie^{dZ_i-d^2/2}$. With $\alpha=0.05$ it fires on positive stores at most an estimated $0.7$ percentage points more often than the best-score gate for $d\in\{1,2,3\}$ and $N\in\{8,64,512,2048\}$ (Table~\ref{tab:anygate}).

\subsection{A cue-swap conformal gate}
\label{sec:swap}

Per-cue control needs the distribution of the cue's best score on a store with nothing relevant, which a deployed assistant cannot observe. It can, however, put the user's earlier messages against the same store. Let $T_0$ be the best score of the current cue on the store $S$ and $T_1,\dots,T_K$ those of $K$ other cues, all computed by the same rule (in our experiments the per-cue $z$-score of recall-v1). The gate speaks when
\[
p=\frac{1+\#\{k\ge1:T_k\ge T_0\}}{K+1}\le\alpha .
\]

\begin{proposition}\label{prop:swap}
If $S$ contains nothing relevant to the current cue and $(T_0,\dots,T_K)$ is exchangeable, then $\Prob(p\le\alpha)\le\lfloor\alpha(K+1)\rfloor/(K+1)\le\alpha$. The bound also holds if there is an exchangeable null vector $(\tilde T_0,\dots,\tilde T_K)$ with $\tilde T_0=T_0$ and $T_k\ge\tilde T_k$ almost surely for $k\ge1$.
\end{proposition}

The second statement covers reference cues that score at least as high as exchangeable null cues would on the same draw. It needs this pathwise coupling: marginal stochastic dominance of each $T_k$ is not enough, because dependence between $T_0$ and the $T_k$ can still make the gate fire too often. The guarantee is marginal over cues and stores. Exchangeability requires the reference cues to come from the same source as the current one, a condition tested in \S\ref{sec:swapexp}. Algorithm~\ref{alg:swap} gives the gate as it would run in an assistant.

\begin{figure}[!h]
\centering
\small
\begin{minipage}{0.95\linewidth}
\hrule height 0.8pt\vspace{3pt}
\refstepcounter{algo}\label{alg:swap}\textbf{Algorithm \thealgo}\quad Cue-swap gate, one decision. Defaults: $\alpha=0.05$, $K=\min(199,\text{earlier messages})$, background of up to $300$ items.
\vspace{3pt}\hrule\vspace{4pt}
\textbf{Input:} current cue $c$; store $S$ of $N$ memories; background $B$, up to $300$ items of the user's history outside $S$ and empty if $S$ holds all of it; reference cues $r_1,\dots,r_K$, the user's earlier messages.
\begin{enumerate}[label=\arabic*:,leftmargin=2.2em]
\item If $(K+1)\alpha<1$ (fewer than $19$ references at $\alpha=0.05$), stay silent.
\item For each cue $x\in\{c,r_1,\dots,r_K\}$, score $x$ against $S\cup B$ in the lexical and the semantic channel, standardise each channel by its median and MAD over $S\cup B$, and let $T(x)$ be the best standardised score over the eligible items of $S$ (Appendix~\ref{app:gates}). In deployment, score each $r_k$ without the session it came from, an asymmetry that Proposition~\ref{prop:swap} does not cover (Appendix~\ref{app:ablation}).
\item $p\leftarrow\big(1+\#\{k: T(r_k)\ge T(c)\}\big)/(K+1)$.
\item If $p\le\alpha$, bring up the memory that attains $T(c)$; otherwise stay silent.
\end{enumerate}
\vspace{2pt}
Median cost on one CPU thread, with the cue encoded, all scores recomputed and $B$ empty: about $0.1$\,s at $N=128$ ($K=128$) and $0.4$\,s at $N=2048$ ($K=199$; Appendix~\ref{app:ablation}).
\vspace{3pt}\hrule height 0.8pt
\end{minipage}
\end{figure}

\section{Experiments}
\label{sec:experiments}

\paragraph{Data.} LongMemEval~\citep{wu2025longmemeval} gives each question a history of about $50$ sessions in the S version and $500$ in the M version, with evidence turns marked; we use user turns as memories and drop abstention questions and the $51$ questions with no labelled evidence turn, leaving a null pool of about $230$ (S) or $2{,}400$ (M) turns per question and $210$ evaluation questions. LoCoMo~\citep{maharana2024locomo} has ten long conversations; memories are all turns, prefixed with the speaker and including image captions, and we drop adversarial questions, leaving $768$ evaluation questions. Store sizes double from $8$ to $2048$ (M) or $256$ (S, LoCoMo), with $10$ (M) or $20$ draws per question and size. A history in S has only about $240$ user turns, so its background is all of the history outside the store, and S is reported up to $N=128$ (Appendix~\ref{app:background}).

\paragraph{Scores.} The lexical channel is tf-idf cosine on word tokens; the semantic channel is the cosine of bge-small-en-v1.5 or bge-large-en-v1.5 embeddings~\citep{xiao2024cpack}. LLM judges are Qwen2.5-7B and 14B-Instruct (AWQ)~\citep{qwen2025qwen25} served with vLLM~\citep{kwon2023vllm}. \emph{Implicit cues} replace each LongMemEval question by a first-person message that a user might send during an everyday task where the memory would help, without asking about it or containing the answer; they were written by Qwen2.5-14B and filtered for leakage, keeping $385$ of $419$ questions (Appendix~\ref{app:details}).

\subsection{False alarms grow with the store}
\label{sec:fa}

\begin{figure}[t]
\centering
\includegraphics[width=\linewidth]{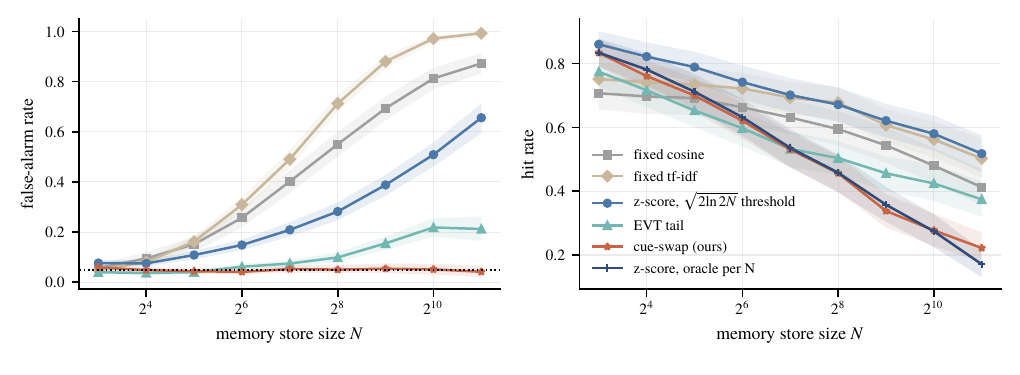}
\caption{False-alarm rate (left) and hit rate (right) against store size on LongMemEval-M with bge-small and explicit questions. The gates without per-size tuning start near $5\%$ (dotted line) at $N=8$. Shaded bands are $95\%$ bootstrap intervals over questions.}
\label{fig:scale}
\end{figure}

Figure~\ref{fig:scale} shows the main result. Tuned to $5\%$ at $N=8$, the cosine threshold reaches $87\%$ false alarms at $N=2048$ and the tf-idf threshold $99\%$; recall-v1, whose threshold grows like $\sqrt{2\ln 2N}$, reaches $66\%$, and the extreme-value gate $21\%$. The same pattern holds with bge-large ($87\%$, $100\%$, $66\%$, $22\%$), with implicit cues ($84\%$, $100\%$, $53\%$, $18\%$) and on LoCoMo up to $N=256$ ($36\%$, $54\%$, $54\%$), where the extreme-value gate stays at $2$--$4\%$. The hit rates of the untuned gates stay high only because they speak almost always.

\subsection{Two channels, two growth rates}
\label{sec:tails}

\begin{figure}[t]
\centering
\includegraphics[width=\linewidth]{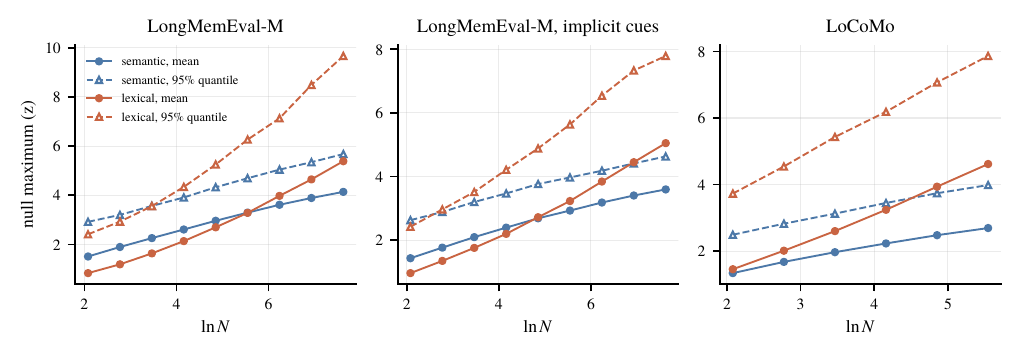}
\caption{Largest null $z$-score in a store against $\ln N$: mean (solid) and $95\%$ quantile (dashed). The lexical maximum grows $1.7$--$2.3$ times as fast as the semantic one at the mean and $2.6$--$2.8$ times at the $95\%$ quantile in the three settings shown.}
\label{fig:tails}
\end{figure}

Figure~\ref{fig:tails} explains why a $\sqrt{2\ln 2N}$ threshold fails. On LongMemEval-M the mean semantic maximum rises from $1.5$ to $4.1$ and is fitted better by $a+b\sqrt{\ln N}$ than by a line in $\ln N$ (residual sums of squares $0.011$ and $0.024$), while the mean lexical maximum rises from $0.8$ to $5.4$ and prefers the line ($0.21$ against $0.73$), as extreme-value theory predicts for Gaussian-type and exponential-type tails and as IR score models assign to non-relevant documents~\citep{manmatha2001modeling,arampatzis2009stop}. A threshold has to track an upper quantile, and at the $95\%$ quantile the lexical maximum rises from $2.4$ to $9.7$ and the semantic one from $2.9$ to $5.7$. The semantic quantile is fitted better by a line in $\ln N$ in some settings and by $\sqrt{\ln N}$ in others, including both implicit-cue settings, so this range of $N$ does not settle its form. Between the smallest and the largest store the lexical quantile grows $2.4$ to $2.8$ times as much as the semantic one in every setting (Appendix~\ref{app:growth}). At $N=2048$ the largest eligible lexical $z$-score alone exceeds the recall-v1 threshold in $62\%$ of null trials on the evaluation half and the semantic one in $22\%$.

\subsection{Theory against data}
\label{sec:theory-check}

\begin{figure}[t]
\centering
\includegraphics[width=\linewidth]{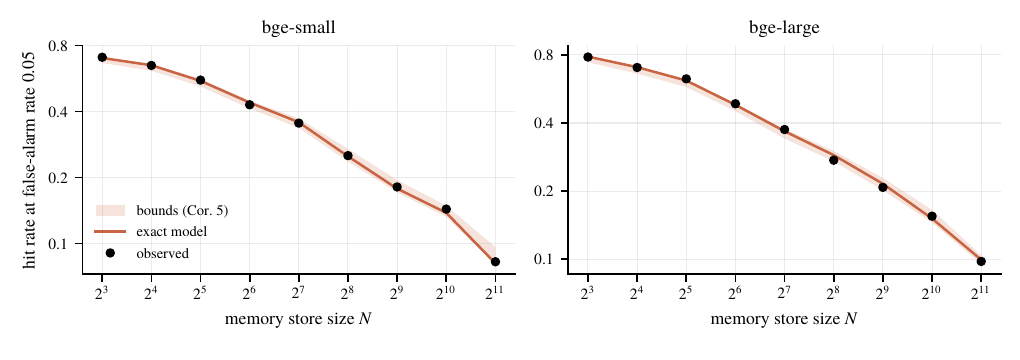}
\caption{Finite-pool consistency check on LongMemEval-M: hit rate of the per-$N$ label-tuned cosine threshold (dots) against the exact finite-pool model at the same threshold, computed per cue from the same pools (line), and the bounds of Corollary~\ref{cor:mixture} (band). Log--log axes.}
\label{fig:theory}
\end{figure}

\paragraph{Consistency.} For the cosine threshold tuned with labels to $5\%$ false alarms at each $N$, we compute the finite-pool hit rate (\S\ref{sec:theory}, after Corollary~\ref{cor:mixture}) from the score distributions of the evaluation cues and compare it with the replay at the same threshold. They agree to within $1.0$ percentage point for bge-small and $1.5$ for bge-large from $N=8$ to $2048$ (Figure~\ref{fig:theory}); the covariance term of Corollary~\ref{cor:mixture} is positive at every $N$ ($0.004$--$0.018$), and the observed hit rates stay within the mixture bounds up to $0.3$ points.

\paragraph{Transfer across questions.} With the threshold and the score distributions taken from the calibration half only, the finite-pool hit rate is the expected replay hit rate of that half, so comparing it with the evaluation half measures how well hit rates carry over to other questions. With bge-small it is $0.32$ $[0.27,0.38]$ against an observed $0.35$ $[0.31,0.41]$ at $N=128$, and $0.09$ $[0.06,0.13]$ against $0.08$ $[0.05,0.12]$ at $N=2048$; the intervals overlap at every $N$, the gap is at most $4.9$ points (at $N=16$), and bge-large behaves alike (Appendix~\ref{app:heldout}). Extrapolating across store sizes is harder. A power law fitted to the observed hit rates at $N=32$--$128$ has exponent $0.33$ and predicts $0.14$ at $N=2048$, where $0.08$ is observed; the effective exponent grows with $N$, so small-store measurements overstate large-store recall. For $N\ge128$ the per-item rate $p_N$ falls below the resolution $1/M$ of a pool of $2{,}400$ scores, so these pools cannot identify the asymptotic tail; on simulated scores with known tails the ratio of the exact hit rate to the limit of Theorem~\ref{thm:power} is $1.0005$ at $N=8$ and at most $1.0001$ from $N=64$ on, and the binormal local exponent approaches $1/s^2$ slowly (Appendix~\ref{app:heldout}).

Once false alarms are controlled, the cosine threshold keeps a hit rate of $71\%$ at $N=8$ and $8\%$ at $N=2048$.

\subsection{What the score buys}
\label{sec:matched}

\begin{table}[t]
\centering
\small
\caption{Hit rate at exactly $5\%$ false alarms (threshold set on the evaluation null trials, ties randomised), LongMemEval-M. Columns add one component at a time: per-cue normalisation of the cosine, the eligibility rules, and the normalised lexical channel (= recall-v1 score).}
\label{tab:matched}
\begin{tabular}{llccccc}
\toprule
cues, embedding & $N$ & raw cosine & semantic $z$ & + eligibility & + lexical $z$ & raw tf-idf \\
\midrule
\multirow{3}{*}{explicit, bge-small} & 8 & .70 & .74 & .74 & .83 & .75 \\
 & 128 & .32 & .38 & .39 & .52 & .50 \\
 & 2048 & .07 & .16 & .16 & .24 & .24 \\
\midrule
\multirow{3}{*}{explicit, bge-large} & 8 & .79 & .84 & .84 & .86 & .75 \\
 & 128 & .34 & .45 & .44 & .58 & .51 \\
 & 2048 & .10 & .17 & .17 & .24 & .21 \\
\midrule
\multirow{3}{*}{implicit, bge-small} & 8 & .36 & .46 & .46 & .63 & .55 \\
 & 128 & .13 & .14 & .14 & .26 & .24 \\
 & 2048 & .05 & .02 & .02 & .06 & .06 \\
\bottomrule
\end{tabular}
\end{table}

Table~\ref{tab:matched} compares scorers at exactly matched false alarms. With explicit questions and bge-small, normalising the cosine per cue adds $4.5$ $[0.6,9.0]$ points at $N=8$, $6.3$ $[-0.4,14.2]$ at $N=128$ and $8.8$ $[-1.0,13.3]$ at $N=2048$; the eligibility rules change the hit rate by at most $0.5$ points; and the normalised lexical channel adds a further $9.8$ $[5.6,13.3]$, $13.7$ $[6.9,22.9]$ and $8.2$ $[-0.2,19.5]$ points (paired by question, with the matched thresholds re-estimated on every bootstrap replicate). Together the deployed score recalls $14$--$26$ points more than raw cosine, with intervals above zero at every $N$. Raw tf-idf alone does about as well as the full score. With implicit cues normalisation helps only on small stores ($9.7$ points at $N=8$), and at $N=2048$ it recalls $2\%$ against $5\%$ for raw cosine, and the lexical channel carries the advantage. On LoCoMo the full score beats raw cosine by $8$--$12$ points with intervals clustered by conversation, $4$--$10$ of them from per-cue normalisation; the lexical channel adds a further $2$--$5$, with an interval above zero only at $N=128$ ($4.8$ $[1.0,9.5]$). With explicit questions at $N=2048$ cosine puts the evidence first in $44\%$ of positive stores, yet no scorer exceeds $24\%$ at matched false alarms, because a best score has to stand out against other cues' best scores as well as the store. With implicit cues at $N=2048$ no scorer reaches $7\%$.

\subsection{The cue-swap gate}
\label{sec:swapexp}

\begin{table}[t]
\centering
\small
\caption{Hit rate (false-alarm rate in parentheses) on LongMemEval-M with the unfiltered-history background. The oracle gates are re-tuned with labels at every $N$ on the calibration half; the cue-swap gate uses $K=199$ other questions and $\alpha=0.05$, with no tuning.}
\label{tab:swap}
\begin{tabular}{llccc}
\toprule
cues, embedding & $N$ & cosine, oracle & $z$-score, oracle & cue-swap \\
\midrule
\multirow{3}{*}{explicit, bge-small} & 8 & .71 (.06) & .83 (.05) & .83 (.06) \\
 & 128 & .35 (.06) & .54 (.05) & .53 (.05) \\
 & 2048 & .08 (.07) & .17 (.02) & .22 (.04) \\
\midrule
\multirow{3}{*}{explicit, bge-large} & 8 & .78 (.05) & .86 (.05) & .86 (.04) \\
 & 128 & .37 (.06) & .56 (.05) & .54 (.05) \\
 & 2048 & .10 (.05) & .18 (.03) & .23 (.05) \\
\midrule
\multirow{3}{*}{implicit, bge-small} & 8 & .41 (.07) & .61 (.04) & .60 (.04) \\
 & 128 & .16 (.06) & .28 (.06) & .27 (.05) \\
 & 2048 & .05 (.07) & .06 (.04) & .05 (.03) \\
\midrule
\multirow{3}{*}{implicit, bge-large} & 8 & .52 (.06) & .63 (.04) & .63 (.04) \\
 & 128 & .21 (.08) & .27 (.05) & .28 (.05) \\
 & 2048 & .07 (.08) & .07 (.03) & .05 (.03) \\
\bottomrule
\end{tabular}
\end{table}

With a background drawn from the user's unfiltered history, the cue-swap gate keeps false alarms between $2.6\%$ and $5.9\%$ at every $N$ in the four LongMemEval settings; no pointwise $95\%$ interval has its lower limit above $5\%$, and the largest simultaneous upper limit over $N$ is $8.8\%$. Its hit rate is within a few points of the label-tuned $z$-score threshold (Table~\ref{tab:swap}). Replacing the labelled null-pool background by the unfiltered history changes the gate's false-alarm rate by at most $1.1$ points and its hit rate by at most $1.4$ points on LongMemEval-M, whereas the extreme-value gate, which fits an absolute tail to the background, turns conservative: its false-alarm rate at $N=2048$ falls from $30\%$ to $21\%$ on LongMemEval-M, and on S its hit rate at $N=8$ falls from $79\%$ to $51\%$ (Appendix~\ref{app:background}).

LoCoMo questions name the speakers, and every memory starts with a speaker's name, so a cue shares names with its own conversation only. With reference cues from other conversations the swap scores are too low and false alarms are $8$--$10\%$; the conversation-clustered $95\%$ intervals exclude $5\%$ at every $N$. With reference cues from the same conversation, the setting a deployed assistant has, false alarms fall to $0.6$--$3.1\%$ and the gate loses power, because other questions of the conversation often have their own evidence in the store; its hit rate falls from $40\%$ at $N=8$ to $3.5\%$ at $N=256$, against $46\%$ to $13\%$ for the label-tuned threshold.

The gate has two parameters (Appendix~\ref{app:ablation}, separate runs with five draws per question and size). With $\alpha=0.01$ the largest false-alarm rate over $N$ is $1.3\%$, with $\alpha=0.05$ it is $5.1$--$6.5\%$ across $K\in\{19,\dots,399\}$, and with $\alpha=0.10$ it is $11$--$13\%$, with the highest values at $N=512$ once $K\ge99$, where the lower $95\%$ limits are about $10\%$. The p-value cannot fall below $1/(K+1)$, so the gate never speaks when $(K+1)\alpha<1$, as at $K=19$ and $49$ with $\alpha=0.01$. A larger $K$ raises the hit rate slightly up to $K\approx100$, and $\alpha$ trades false alarms for hits: at $N=2048$ with $K=199$ the hit rate is $8\%$, $22\%$ and $28\%$ for the three levels.

\subsection{Chronological replay}
\label{sec:temporal}

\begin{table}[t]
\centering
\small
\caption{Chronological replay on LongMemEval-M (bge-small, evaluation half). Top: share of session-opening turns on which each gate speaks, with the store holding every earlier user turn ($630$ cues per size); each cue opens one of the first three sessions at which the history reaches the target size, with the median actual size in parentheses. Bottom: hit rate of the benchmark question asked after the last session ($N\approx2{,}430$), for explicit questions and implicit cues. The fixed and transferred thresholds were tuned on random null stores from the calibration half, scored as here: store-only tf-idf weights, with the store as null sample.}
\label{tab:temporal}
\begin{tabular}{lcccccc}
\toprule
target $N$ & fixed cosine & fixed tf-idf & recall-v1 & extreme-value & $z$-score, transferred & cue-swap \\
\midrule
8 (16) & .25 & .17 & .12 & .00 & .03 & .01 \\
16 (23) & .31 & .19 & .13 & .00 & .03 & .06 \\
32 (39) & .38 & .18 & .16 & .01 & .04 & .06 \\
128 (135) & .64 & .32 & .23 & .02 & .04 & .08 \\
512 (519) & .73 & .52 & .39 & .13 & .04 & .06 \\
2048 (2055) & .87 & .78 & .57 & .33 & .06 & .07 \\
\midrule
hit, explicit & .38 & .45 & .45 & .34 & .19 & .20 \\
hit, implicit & .16 & .22 & .22 & .14 & .03 & .02 \\
\bottomrule
\end{tabular}
\end{table}

The replays so far draw stores at random and take reference cues from other questions. A deployed assistant sees its store grow in time and can compare a cue only with the user's own earlier messages. We therefore order each LongMemEval-M history by session date and let the store hold every user turn of the earlier sessions. The null sample for standardisation is the store itself, and the cue-swap gate takes $K=\min(199,n)$ reference cues from the user's own $n$ earlier turns, each compared with the store without its own session. The cues are the opening turns of the first three sessions at which the store reaches each size. The tf-idf term weights come from the store alone. Every firing on these cues counts as an interruption, whether or not the memory it brings up would help.

The fixed and transferred thresholds were tuned on benchmark questions, so Table~\ref{tab:temporal} also tests how they carry over to opening turns. The fixed thresholds start above their nominal $5\%$ and grow as in the random replays: at $N=2048$ the cosine threshold speaks on $87\%$ of opening turns and the tf-idf threshold on $78\%$, recall-v1 on $57\%$ and the extreme-value gate on $33\%$. The transferred $z$-score threshold, which needs labels and a table of thresholds by size, stays at $3$--$6\%$. The cue-swap gate speaks on $6$--$8\%$ of opening turns at every target size from $N=16$ (median $N=23$) on, with pointwise $95\%$ intervals between $4.0\%$ and $10.0\%$; at $N=8$ most histories offer fewer than $19$ earlier turns, the minimum for $\alpha=0.05$, so it rarely speaks. Asked the benchmark question after the last session, it finds the evidence in $20\%$ of cases with explicit questions and $2\%$ with implicit cues, close to the $z$-score threshold tuned with labels on random null stores ($19\%$ and $3\%$). Asked right after the last evidence session (median $N=1{,}840$), it finds $25\%$ and $3\%$. recall-v1 finds $45\%$ and $22\%$ at the price of speaking on $57\%$ of opening turns at $N=2048$, and the extreme-value gate finds $34\%$ and $14\%$, close to the fixed cosine threshold, while speaking on a third of them.

Opening turns and benchmark questions are different cues, so we also asked the benchmark questions on chronological stores with every answer session removed, at the session boundaries where the store first reaches $N=32$, $128$, $512$ and $2048$ and after the last session (Appendix~\ref{app:ablation}). The cue-swap gate fires on $5$--$8\%$ of explicit questions, with $95\%$ intervals between $2.4\%$ and $12.4\%$, and on $2$--$7\%$ of implicit cues; recall-v1 fires on $12$--$71\%$ and $10$--$60\%$.

\subsection{Several relevant memories}
\label{sec:sensitivity}

Benchmark stores hold one relevant memory, and we varied this on LongMemEval-M (Appendix~\ref{app:sensitivity}). With every evidence turn of the question in the store, hit rates rise because any of them can be brought up; the cue-swap gate recalls $29\%$ at $N=2048$ against $22\%$. A fact repeated three times changes the gates that use the largest score by at most $2$ points, but the extreme-value gate loses up to $20$ points ($40\%$ against $60\%$ at $N=64$), because the copies thicken the upper tail it fits. With the whole evidence session in the store, the turns around the answer often outscore it: at $N=2048$ the cue-swap gate brings up the answer turn in $19\%$ of positive stores and some turn of the evidence session in $26\%$.

\subsection{A prospective-memory benchmark}
\label{sec:triggerbench}

TriggerBench~\citep{zhang2026triggerbench} plants a constraint in a conversation and ends it with a trigger turn that is harmless on its own but calls for a reminder given the constraint. We treat the user turns before the trigger as the store, the constraint turn as the evidence and the trigger as the cue, with $4$--$9$ user turns per conversation on average in the base split and up to about $640$ in the 40K split, and stores grown to $N=2048$ with filler user turns from the 40K split (Appendix~\ref{app:triggerbench}). The null condition removes the constraint turn; the negative condition is the benchmark's own control, where the constraint has been resolved or does not apply.

Over the natural base stores the constraint has the highest cosine in $71\%$ of logical-adherence items and $67\%$ of safe-coding items, but in $37\%$ of temporal-grounding and $13\%$ of state-tracking items, where the conversation around the constraint shares more with the trigger than the constraint does. At about $4\%$ false alarms in the null condition, the label-tuned $z$-score threshold recalls the constraint in $5$--$9\%$ of positive items across store sizes, and it fires on $14$--$17\%$ of the resolved negatives, more often than on the null stores. A similarity gate measures topical relevance and cannot see whether a constraint is still in force. The cue-swap gate with triggers of other conversations as references fires on $16$--$37\%$ of null stores, because the conversation turns share the topic of their own trigger and not of the references, as with reference cues from other LoCoMo conversations.

\subsection{LLM judges and a reranker}
\label{sec:llm}

\begin{figure}[t]
\centering
\includegraphics[width=\linewidth]{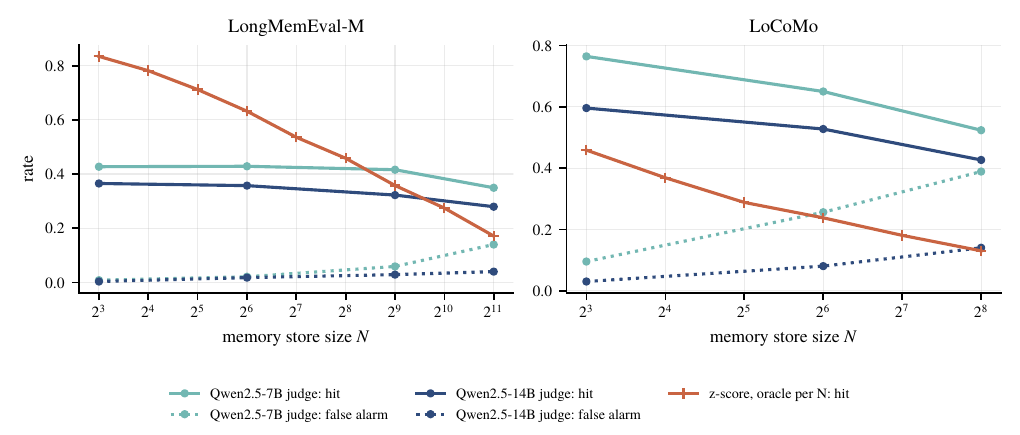}
\caption{LLM judges reading the five most similar memories. Solid lines are hit rates, dotted lines false-alarm rates; the orange line is the hit rate of the label-tuned $z$-score threshold at about $5\%$ false alarms.}
\label{fig:llm}
\end{figure}

A judge that reads the candidates can tell a near miss from a relevant memory, and Figure~\ref{fig:llm} shows that with explicit questions it leads on large stores. On LongMemEval-M at $N=2048$ the 14B judge has $28\%$ $[23,34]$ hits at $4\%$ false alarms, against $17$--$22\%$ for the best similarity gates, with $[17,27]$ for the cue-swap gate. It is conservative on small stores ($37\%$ hits at $N=8$ against $83\%$) and its false-alarm rate still grows with the store, from $0.3\%$ to $4\%$ on LongMemEval-M and from $3\%$ to $14\%$ on LoCoMo, roughly linearly in $\ln N$. The 7B judge has more hits and more false alarms ($14\%$ at $N=2048$; $39\%$ on LoCoMo at $N=256$). With implicit cues the 14B judge recalls $25\%$ at $N=8$ and $15\%$ at $N=2048$, while its false alarms climb from $0\%$ to $11\%$. Retrieval becomes a second bottleneck. At $N=2048$ the evidence is among the top five in $65\%$ of positive trials with explicit questions and in $36\%$ with implicit cues.
A cross-encoder reranker (bge-reranker-base~\citep{xiao2024cpack}) reads the cue with each of the $20$ memories closest in cosine and speaks if the best reranker score clears a threshold. Tuned to $5\%$ false alarms at $N=8$, it fires on $81\%$ of null stores at $N=2048$. At exactly $5\%$ false alarms it recalls $80\%$, $57\%$, $31\%$ and $23\%$ at $N=8$, $64$, $512$ and $2048$ with explicit questions, close to the deployed $z$-score ($83\%$, $66\%$, $35\%$ and $24\%$), and $46\%$, $11\%$, $5\%$ and $2\%$ with implicit cues, well below it ($63\%$, $36\%$, $15\%$ and $6\%$). It ranks better than cosine, putting the evidence first in $58\%$ of positive stores at $N=2048$ against $42\%$ for cosine in the same run, and its threshold still has to climb with the null maximum.

\subsection{What false alarms are}
\label{sec:annotation}

The false-alarm rate counts every firing on a null store as an error, but a null store can contain memories that are useful without holding the labelled answer. We sampled $240$ recall-v1 false alarms on LongMemEval-M, $80$ each at $N=64$, $512$ and $2048$, and mixed in $30$ evidence turns. Three annotators labelled each item blind as \emph{helpful} (a response would use its content, including a partial answer), \emph{same topic but not useful}, or \emph{unrelated}: one of the authors, who read the English text with a machine translation alongside, and two LLMs, Claude and GPT (Appendix~\ref{app:annotation}). Agreement is computed on $229$ items: the $232$ labelled after a translation fix, minus three the human marked undecidable. Cohen's $\kappa$~\citep{cohen1960kappa} is $0.62$ between the human and Claude, $0.54$ between the human and GPT and $0.74$ between the two models; Fleiss' $\kappa$~\citep{fleiss1971kappa} is $0.63$, and on the $203$ false alarms among them it is $0.61$, or $0.47$ for helpful against not helpful. The majority marks $26$ of the $30$ evidence turns helpful.

Of the $206$ false alarms labelled after the fix, $3\%$, $5\%$ and $12\%$ at $N=64$, $512$ and $2048$ are helpful by majority vote, with Wilson $95\%$ intervals of $[1,10]$, $[2,13]$ and $[7,22]$. Where the human disagrees with both models the majority follows the models; the human alone, the most lenient annotator, puts the figure at $9$--$16\%$. Same-topic distractors are $45\%$ at $N=64$ and $53$--$67\%$ at larger $N$.

\section{Design implications}

For builders of long-lived assistants the results suggest four rules.
\begin{enumerate}[leftmargin=1.6em]
\item \textbf{Standardise each channel per cue and combine the channels.} At matched false alarms with explicit questions this recalled $14$--$26$ points more than raw cosine with bge-small and $7$--$24$ with bge-large. The eligibility rules of recall-v1 changed the hit rate by less than $1$ point and can be dropped.
\item \textbf{Let the threshold grow with the store, separately per channel.} A threshold tuned once at $N=8$ reached $87$--$99\%$ false alarms at $N=2048$, and the lexical maximum grew about twice as fast as the semantic one, faster still in its upper tail.
\item \textbf{Calibrate against the user's own earlier messages with the cue-swap gate} (Algorithm~\ref{alg:swap}). It needs no labels and about $0.4$\,s per decision at $N=2048$, and at $\alpha=0.05$ it can speak once the user has sent $19$ earlier messages. In the chronological replay it spoke on $6$--$8\%$ of session openings, where recall-v1 spoke on $57\%$ at $N=2048$.
\item \textbf{Expect the usable store size to be limited by separation.} Do not extrapolate it from small stores, whose slopes overstate large-store recall, and expect to lose more of it to the comparison across cues than to ranking. In the binormal model the capacity bound of Corollary~\ref{cor:capacity} grows exponentially in the squared separation, and by Theorem~\ref{thm:anygate} the store size that any gate can serve is capped by the separation of a single comparison. Better representations, richer cues and LLM judges that read candidates raise the cap by raising that separation. A constraint that only reasoning connects to the cue, as in TriggerBench, is out of reach of a similarity gate.
\end{enumerate}

\section{Limitations}
The replays use benchmark questions and model-written implicit cues in place of real users' messages, and the implicit cues were written by a model that saw the answer. LoCoMo has ten conversations, which appear in both halves, and intervals clustered on ten are rough. LongMemEval assembles every history from one pool of sessions written for several simulated users, so the calibration and evaluation halves share history text and our held-out tests hold out questions, not users. The main replays draw random stores with one relevant item.

The theory assumes independent memories and one relevant memory, and Theorem~\ref{thm:anygate} says nothing when the chi-square divergence is infinite, as in the binormal model with $s^2\ge2$. The guarantee of the cue-swap gate is marginal and depends on exchangeability, which reference cues from another conversation or scenario violated on LoCoMo and TriggerBench. In the chronological replay each reference cue is scored without its own session, so its best score is taken over fewer memories than the current cue's and the guarantee does not cover this variant; with independent scores and the session sizes of LongMemEval-M the simulated false-alarm rate rises by about one point at $N=16$ and by nothing measurable from $N=128$ (Appendix~\ref{app:ablation}).

Annotations come from one human, who is not a native English speaker and read machine translations, and two LLMs; the false alarms were sampled under the version-1 protocol. All experiments are in English.

\paragraph{Use of AI assistants.} Claude (Anthropic) and GPT (OpenAI) were used to write code, run analyses, help draft and edit the text and review the proofs, and served as the two LLM annotators of \S\ref{sec:annotation}. The authors checked the proofs and the reported numbers and take responsibility for the content.

\bibliographystyle{plainnat}
\bibliography{refs}

\appendix

\section{Proofs}
\label{app:proofs}

\begin{lemma}\label{lem:transform}
Under (A1)--(A2), $U_1,\dots,U_N$ are i.i.d.\ uniform on $(0,1)$ and independent of $U$; $\Prob(U\le u)=R(u)$ for $u\in(0,1)$, where $R$ is extended to the distribution function of $U$ on $\mathbb R$, with $R(u)=0$ for $u<0$ and $R(0)=\Prob(U=0)=R(0+)$; up to a null event, $\{Y\ge\Fbar^{-1}(u)\}=\{U\le u\}$ and $\{Y>\max_iX_i\}=\{U<\min_iU_i\}$.
\end{lemma}

\begin{proof}
Continuity of $F$ makes $\Fbar(X_i)$ uniform. By (A1), $\Fbar$ is continuous and strictly decreasing on the support, so $\{t:\Fbar(t)\le u\}=[\Fbar^{-1}(u),\infty)$ and $\Prob(U\le u)=\Prob(Y\ge\Fbar^{-1}(u))=\Gbar(\Fbar^{-1}(u))=R(u)$, using continuity of $G$. Right-continuity gives $R(0)=R(0+)$; $U$ has an atom at $0$ exactly when $Y$ exceeds the upper end of the support of $F$ with positive probability. For the last identity let $M=\max_iX_i$, which lies in the support almost surely. If $Y$ is in the support, $Y>M$ is equivalent to $\Fbar(Y)<\Fbar(M)=\min_iU_i$ by strict monotonicity. If $Y$ is above the support, $U=0<\min_iU_i$ and $Y>M$ both hold almost surely; if $Y$ is below it, $U=1>\min_iU_i$ and $Y<M$ both hold almost surely.
\end{proof}

\begin{proof}[Proof of Proposition~\ref{prop:sandwich}]
By Lemma~\ref{lem:transform} the hit event is $\{U\le p_N,\ U<\min_{i<N}U_i\}$ up to a null event. Given $U=u$ the second condition has probability $(1-u)^{N-1}$, and integrating against the law of $U$ gives the integral, read as a Lebesgue--Stieltjes integral that includes a possible atom at $0$. On $[0,p_N]$ the integrand lies between $(1-p_N)^{N-1}=F(t_N)^{N-1}=(1-\alpha)^{(N-1)/N}$ and $1$, which gives the bounds. Bernoulli's inequality gives $(1-\alpha)^{1/N}\le1-\alpha/N$, hence $p_N\ge\alpha/N$; and $(1-\alpha)^{1/N}=e^{-\lambda/N}\ge1-\lambda/N$ gives $p_N\le\lambda/N$.
\end{proof}

\begin{proof}[Proof of Theorem~\ref{thm:power}]
Substituting $u=p_Nv$ gives $\Hit_N/R(p_N)=\int_{[0,1]}f_N\,d\mu_N$ with $f_N(v)=(1-p_Nv)^{N-1}$ and $\mu_N(v)=R(p_Nv)/R(p_N)$, a distribution function on $[0,1]$ because $\kappa>0$ forces $R(0+)=0$. Regular variation gives $\mu_N(v)\to v^\kappa$ for every $v$, and the limit is continuous, so $\mu_N$ converges weakly to $\mu(dv)=\kappa v^{\kappa-1}dv$. Since $Np_N\to\lambda$ and $\ln f_N(v)=(N-1)\ln(1-p_Nv)=-(N-1)p_Nv+O(Np_N^2)$, $f_N\to f(v)=e^{-\lambda v}=(1-\alpha)^v$ uniformly on $[0,1]$. Then $|\int f_Nd\mu_N-\int fd\mu|\le\sup|f_N-f|+|\int fd\mu_N-\int fd\mu|\to0$, the second term by weak convergence and the boundedness and continuity of $f$. The integrand $(1-\alpha)^v$ lies in $[1-\alpha,1]$ and $\mu$ is a probability measure, so $C\in[1-\alpha,1]$. Finally $R(p_N)=p_N^\kappa\ell(p_N)$ with $p_N^\kappa\sim\lambda^\kappa N^{-\kappa}$, and $p_N/(1/N)\to\lambda\in(0,\infty)$ gives $\ell(p_N)/\ell(1/N)\to1$ by the uniform convergence theorem for slowly varying functions~\citep[Thm.~1.2.1]{bingham1987regular}.
\end{proof}

\begin{proof}[Details of Example~\ref{ex:binormal}]
By the Mills-ratio expansion, $\ln R=-(t-d)^2/(2s^2)-\ln((t-d)/s)-\frac12\ln2\pi+o(1)$ and $\ln p=-t^2/2-\ln t-\frac12\ln2\pi+o(1)$. Subtracting $s^{-2}\ln p$ gives $\ln\ell(p)=(d/s^2)t+(1/s^2-1)\ln t+c+o(1)$ with $c=-d^2/(2s^2)+\ln s+\frac12(1/s^2-1)\ln2\pi$. For fixed $\lambda'>0$ let $t_1=t(\lambda'p)$ and $t_2=t(p)$. Since $t^2\sim2\ln(1/p)$, $t_1/t_2\to1$; then $\ln\lambda'=-(t_1^2-t_2^2)/2-\ln(t_1/t_2)+o(1)$ gives $(t_1-t_2)(t_1+t_2)\to-2\ln\lambda'$, so $t_1-t_2=O(1/t)\to0$ and $\ell(\lambda'p)/\ell(p)\to1$. The local exponent follows from $d\ln R/dt=-\phi(z)/(s\bar\Phi(z))$ with $z=(t-d)/s$ and $d\ln p/dt=-\phi(t)/\bar\Phi(t)$, and, as $t\to\infty$ with $d,s$ fixed, $h(x)=x+1/x+O(x^{-3})$ gives $\kappa_{\mathrm{eff}}=s^{-2}(1-d/t)(1+O(t^{-2}))=s^{-2}(1-d/t)+O(t^{-2})$. The example in the text is computed from the exact expression.
\end{proof}

\begin{proof}[Proof of Corollary~\ref{cor:capacity}]
$R(p)=\bar\Phi((\bar\Phi^{-1}(p)-d)/s)$ is continuous and strictly increasing, $R(q_\beta)=\beta$, and $p_N$ is strictly decreasing in $N$ with $p_{N^*}=q_\beta$. If $N>N^*(\beta)$ then $p_N<q_\beta$ and $\Hit_N\le R(p_N)<\beta$; if $N\le N^*(\beta')$ then $R(p_N)\ge\beta'$ and $\Hit_N\ge(1-\alpha)R(p_N)\ge\beta$. For the expansion write $\ln N^*=\ln\lambda-\ln(-\ln(1-q_\beta))$. From $q\le-\ln(1-q)\le q/(1-q)$, $-\ln(-\ln(1-q))=-\ln q+\delta$ with $\delta\in[\ln(1-q),0]$. From the Mills-ratio bounds $t\phi(t)/(1+t^2)\le\bar\Phi(t)\le\phi(t)/t$ of~\citet{gordon1941mills}, $-\ln q=t^2/2+\ln t+\frac12\ln2\pi+\varepsilon$ with $\varepsilon\in[0,\ln(1+t^{-2})]$. Hence $\eta=\delta+\varepsilon$ lies in the stated interval.
\end{proof}

\begin{proof}[Proof of Corollary~\ref{cor:mixture}]
For each cue the hit event lies in $\{Y_c\ge\tau\}$, which gives the upper bound for every cue and hence for the average. It contains $\{Y_c\ge\tau\}\cap\{\text{all other items}<\tau\}$; the evidence item is drawn independently of the others, so $\Hit_c\ge\Gbar_c(\tau)\,\Prob(\max\text{ of }N-1\text{ items}<\tau)$. Under i.i.d.\ sampling the last probability is $F_c(\tau-)^{N-1}\ge F_c(\tau-)^N=1-\alpha_c$. Under sampling without replacement from a pool of $M$ items, $L$ of them below $\tau$, it is $\binom{L}{N-1}/\binom{M}{N-1}\ge\binom{L}{N}/\binom{M}{N}=1-\alpha_c$, because the ratio of the two sides is $(L-N+1)/(M-N+1)\le1$ when $L\ge N-1$, and both sides are zero otherwise. In both cases $\Hit_c\ge\Gbar_c(\tau)(1-\alpha_c)$. Averaging and writing $\E[\alpha_c\Gbar_c]=\E\alpha_c\,\E\Gbar_c+\mathrm{Cov}(\alpha_c,\Gbar_c)$ gives the identity. No continuity is needed because the event requires the other items to be strictly below $\tau$.
\end{proof}

\begin{proof}[Proof of Theorem~\ref{thm:anygate}]
Averaging over the gate's randomisation and other inputs, which have the same law on null and positive stores and are independent of the store and of $J$, let $\phi_j(z)\in[0,1]$ be the probability that the gate brings up position $j$ given the store $z$, and $\phi=\sum_j\phi_j\le1$ the probability that it speaks. Let $\Prob_0$ and $\Prob_1$ be the laws of a null and a positive store, $\Prob_1^{(j)}$ the law of a positive store given $J=j$, and $\chi^2_j=\chi^2(P_j\,\|\,Q_j)$. Then $\alpha=\E_0\phi$.

(i) If $\chi^2_j=\infty$ for some $j\in S$, the right side is infinite. Otherwise $\Prob_1^{(j)}$ has density $L_j(Z_j)$ with respect to $\Prob_0$ for $j\in S$, and $\Prob_1$ has density $\Lambda=\sum_{j\in S}\pi_jL_j(Z_j)$. Under $\Prob_0$ the $L_j(Z_j)$ are independent with mean $1$ and second moment $1+\chi^2_j$, so
\[
\E_0\Lambda^2=\sum_{j\in S}\pi_j^2(1+\chi^2_j)+\sum_{i\ne j\in S}\pi_i\pi_j=1+\sum_{j\in S}\pi_j^2\chi^2_j ,
\]
that is, $\mathrm{Var}_0\Lambda=\sum_{j\in S}\pi_j^2\chi^2_j$. With $\E_0\Lambda=1$,
\[
\E_1\phi-\alpha=\E_0[(\Lambda-1)\phi]=\E_0[(\Lambda-1)(\phi-\alpha)]\le\sqrt{\mathrm{Var}_0\Lambda\ \mathrm{Var}_0\phi}
\]
by the Cauchy--Schwarz inequality, and $\mathrm{Var}_0\phi=\E_0\phi^2-\alpha^2\le\alpha-\alpha^2$ because $0\le\phi\le1$. A hit requires speaking, so $\Hit\le\E_1\phi$. In the uniform case $\sum_j\pi_j^2\chi^2_j\le\chi^2/N$. For the binormal divergence, completing the square gives $\int\varphi_{d,s}^2/\varphi_{0,1}=(s\sqrt{2-s^2})^{-1}e^{d^2/(2-s^2)}$ for the normal densities $\varphi$, finite exactly when $s^2<2$.

(ii) If $\alpha=0$, then $\phi=0$ almost surely under $\Prob_0$, hence under $\Prob_1\ll\Prob_0$, and $\Hit\le\E_1\phi=0$. Let $0<\alpha\le1$ and $M=\max_{j\in S}\pi_jL_j(Z_j)$, which has $\E_0M\le\sum_{j\in S}\E_0L_j(Z_j)\le N$. Absolute continuity alone gives $\Prob_1^{(j)}$ the density $L_j(Z_j)$ with respect to $\Prob_0$, so $\Hit=\sum_{j\in S}\pi_j\E_0[\phi_jL_j(Z_j)]$. For every $k\ge0$,
\[
\Hit-k\,\E_0\phi=\E_0\Big[\sum_{j\in S}\phi_j\big(\pi_jL_j(Z_j)-k\big)-k\sum_{j\notin S}\phi_j\Big]\le\E_0(M-k)_+ ,
\]
with equality for the gate that brings up the arg-max when $M>k$, does so with some probability when $M=k$, and is silent otherwise. Take $k=\inf\{x\ge0:\Prob_0(M>x)\le\alpha\}$; then $\Prob_0(M>k)\le\alpha\le\Prob_0(M\ge k)$, and speaking at $M=k$ with probability $(\alpha-\Prob_0(M>k))/\Prob_0(M=k)$ when $\Prob_0(M=k)>0$ gives false-alarm rate $\alpha$. Its hit rate $\Hit^*$ then satisfies $\Hit^*-k\alpha=\E_0(M-k)_+$, and every gate with false-alarm rate $\alpha'\le\alpha$ has $\Hit\le\Hit^*-k(\alpha-\alpha')\le\Hit^*$.
\end{proof}

\begin{proof}[Proof of Proposition~\ref{prop:swap}]
Define $p_j$ for each $j\in\{0,\dots,K\}$ by treating $T_j$ as the current score, and let $m=\lfloor\alpha(K+1)\rfloor$. The condition $p_j\le\alpha$ is $\#\{i\ne j:T_i\ge T_j\}\le m-1$. At most $m$ indices satisfy it: among $m+1$ such indices, the one with the smallest $T$ would have $m$ others at least as large. By exchangeability all $\Prob(p_j\le\alpha)$ are equal, so $\Prob(p_0\le\alpha)=(K+1)^{-1}\E\sum_j\mathbf 1\{p_j\le\alpha\}\le m/(K+1)$. For the second statement, $p(T)\ge p(\tilde T)$ pointwise under the coupling, because raising reference scores can only raise the count in $p$; hence $\Prob(p(T)\le\alpha)\le\Prob(p(\tilde T)\le\alpha)\le m/(K+1)$.
\end{proof}

\section{Gates and scores}
\label{app:gates}

\paragraph{Lexical channel.} Tokens are lower-cased latin words of at least two characters and numbers of at least three digits. A term with count $n$ in a text has weight $(1+\ln n)\ln\big(1+(D+1)/(\mathrm{df}+0.5)\big)$, where document frequencies are computed over all items of the data set; the score is the cosine of weight vectors, and we also record the number of shared terms.

\paragraph{recall-v1.} For a cue, the null sample is the store together with up to $300$ background items, drawn after the store from the turns of the same history that are not in the store: the user turns of the question's history in LongMemEval, sessions with labelled evidence included, and all turns of the conversation in LoCoMo. Each channel is standardised by the median and $1.4826$ times the median absolute deviation of the null sample, with spread floors $0.04$ (lexical) and $0.03$ (semantic). A memory is eligible in the lexical channel if its tf-idf cosine is at least $0.1$ with at least two shared terms, and in the semantic channel if its cosine is at least $0.55$. The gate speaks when the largest eligible $z$-score reaches $\max(2,\sqrt{2\ln 2N})+0.5$. These are the defaults of the deployed plugin; its activation noise, recency bonus and per-turn throttles are switched off.

\paragraph{Extreme-value gate.} For each channel we take the null sample above its $90$th percentile (at least ten points), fit an exponential tail to the excesses and compute a tail p-value for every memory; the gate speaks if the smallest p-value over memories and channels is at most $\alpha/(2N)$.

\paragraph{Cue-swap gate.} The reference cues are $K=199$ questions other than the current one (LongMemEval) or other questions of the same conversation, up to $199$ (LoCoMo, same-source variant). Their best $z$-scores are computed with the same background items and the same eligibility rules.

\paragraph{LLM judges.} The prompt shows the cue and the five memories with the highest cosine (truncated to $400$ characters) and asks for the number of the single note to bring up or \texttt{NONE}; decoding is greedy.

\section{Background drawn from the null pool}
\label{app:background}

\begin{table}[h]
\centering
\small
\caption{False-alarm rate / hit rate on the evaluation half with the background drawn from the user's unfiltered history (full, the protocol of the main text) or from the labelled null pool (pool). bge-small; LoCoMo with reference cues from other conversations.}
\label{tab:background}
\begin{tabular}{llcccc}
\toprule
setting & $N$ & background & recall-v1 & extreme-value & cue-swap \\
\midrule
\multirow{6}{*}{M, explicit} & \multirow{2}{*}{8} & pool & .07 / .85 & .05 / .80 & .05 / .83 \\
 & & full & .08 / .86 & .04 / .77 & .06 / .83 \\
 & \multirow{2}{*}{128} & pool & .20 / .69 & .10 / .57 & .05 / .53 \\
 & & full & .21 / .70 & .07 / .53 & .05 / .53 \\
 & \multirow{2}{*}{2048} & pool & .65 / .52 & .30 / .40 & .05 / .22 \\
 & & full & .66 / .52 & .21 / .37 & .04 / .22 \\
\midrule
\multirow{6}{*}{M, implicit} & \multirow{2}{*}{8} & pool & .04 / .62 & .04 / .58 & .04 / .61 \\
 & & full & .04 / .62 & .04 / .55 & .04 / .60 \\
 & \multirow{2}{*}{128} & pool & .16 / .38 & .07 / .29 & .05 / .25 \\
 & & full & .15 / .38 & .07 / .27 & .05 / .27 \\
 & \multirow{2}{*}{2048} & pool & .52 / .24 & .24 / .17 & .03 / .05 \\
 & & full & .53 / .23 & .18 / .14 & .03 / .05 \\
\midrule
\multirow{4}{*}{S, explicit} & \multirow{2}{*}{8} & pool & .06 / .85 & .03 / .79 & .05 / .82 \\
 & & full & .06 / .83 & .01 / .51 & .05 / .80 \\
 & \multirow{2}{*}{128} & pool & .19 / .70 & .08 / .54 & .05 / .55 \\
 & & full & .16 / .68 & .00 / .17 & .05 / .54 \\
\midrule
\multirow{4}{*}{LoCoMo} & \multirow{2}{*}{8} & pool & .11 / .59 & .04 / .49 & .09 / .55 \\
 & & full & .11 / .58 & .02 / .45 & .09 / .55 \\
 & \multirow{2}{*}{256} & pool & .55 / .35 & .07 / .18 & .08 / .18 \\
 & & full & .54 / .35 & .04 / .15 & .08 / .18 \\
\bottomrule
\end{tabular}
\end{table}

The version-1 protocol drew the background from the labelled null pool, which no deployed assistant can do. Table~\ref{tab:background} compares the two protocols; in both the background is drawn after the store and is disjoint from it. The unfiltered history contains the evidence sessions and other turns on the topic of the cue, so its upper tail is heavier. recall-v1 is hardly affected, because the median and the MAD are robust to a few high scores. The cue-swap gate changes by at most $1.1$ points in false-alarm rate and $1.4$ points in hit rate on LongMemEval-M, and by at most $2.2$ points on S, since the current cue and the reference cues are standardised on the same background. The extreme-value gate fits an absolute tail to the background, so its p-values grow with the contamination: its false-alarm rate falls and its hit rate falls with it, sharply on S, where the evidence sessions are a large share of a short history.

A question's whole history in LongMemEval-S has about $240$ user turns, so the S background is all of the history outside the store: $235$ turns on average at $N=8$ and $115$ at $N=128$. Only ten S evaluation questions have a null pool of at least $256$ turns, so S is reported up to $N=128$. On M the background has $300$ items except in a few trials at $N=2048$ (mean $298$, minimum $221$); on LoCoMo the mean is $296$ at $N=128$ and $272$ at $N=256$.

\section{Cue-swap gate: number of reference cues and nominal level}
\label{app:ablation}

\begin{table}[h]
\centering
\small
\caption{Cue-swap gate on LongMemEval-M with the unfiltered-history background (bge-small, $210$ evaluation questions, $5$ draws per question and size, a run separate from Table~\ref{tab:swap}). ``max FA'' is the largest false-alarm rate over $N\in\{8,64,512,2048\}$, with the largest upper $95\%$ bootstrap limit in parentheses. With implicit cues ($194$ questions, $K\le199$) the largest false-alarm rates are $1.2\%$, $5.3\%$ and $10.5\%$ for the three levels.}
\label{tab:ablation}
\begin{tabular}{ccccccc}
\toprule
$\alpha$ & $K$ & max FA & hit, $N=8$ & hit, $N=64$ & hit, $N=512$ & hit, $N=2048$ \\
\midrule
\multirow{5}{*}{0.01} & 19 & .000 (.000) & .00 & .00 & .00 & .00 \\
 & 49 & .000 (.000) & .00 & .00 & .00 & .00 \\
 & 99 & .011 (.021) & .59 & .37 & .15 & .07 \\
 & 199 & .013 (.021) & .63 & .39 & .17 & .08 \\
 & 399 & .013 (.020) & .64 & .40 & .18 & .09 \\
\midrule
\multirow{5}{*}{0.05} & 19 & .065 (.081) & .77 & .58 & .31 & .18 \\
 & 49 & .051 (.066) & .78 & .58 & .30 & .18 \\
 & 99 & .061 (.084) & .82 & .62 & .33 & .21 \\
 & 199 & .060 (.080) & .83 & .63 & .34 & .22 \\
 & 399 & .060 (.082) & .84 & .63 & .33 & .22 \\
\midrule
\multirow{5}{*}{0.10} & 19 & .112 (.136) & .86 & .68 & .42 & .27 \\
 & 49 & .113 (.140) & .88 & .70 & .44 & .28 \\
 & 99 & .127 (.157) & .88 & .70 & .44 & .28 \\
 & 199 & .128 (.161) & .89 & .71 & .45 & .28 \\
 & 399 & .132 (.166) & .89 & .71 & .45 & .28 \\
\bottomrule
\end{tabular}
\end{table}

\paragraph{Intervals.} Table~\ref{tab:swapci} gives the false-alarm rate of the cue-swap gate with bootstrap intervals over questions (LongMemEval) or conversations (LoCoMo, ten clusters). Rates are pooled over trials, as in the gate tables, and each replicate takes the ratio of summed firings to summed trials over the resampled clusters. Simultaneous limits over the store sizes use a Bonferroni correction.

\begin{table}[h]
\centering
\small
\caption{False-alarm rate of the cue-swap gate ($K=199$, $\alpha=0.05$, bge-small) with pointwise $95\%$ intervals; the last row gives the largest simultaneous upper limit over $N$.}
\label{tab:swapci}
\begin{tabular}{lcccc}
\toprule
$N$ & M, explicit & M, implicit & LoCoMo, other conv. & LoCoMo, same conv. \\
\midrule
8 & .059 [.049, .070] & .041 [.033, .050] & .089 [.079, .103] & .031 [.027, .035] \\
64 & .041 [.032, .052] & .048 [.037, .060] & .098 [.081, .124] & .013 [.010, .016] \\
256 & .050 [.035, .068] & .045 [.032, .060] & .083 [.066, .108] & .006 [.002, .011] \\
2048 & .042 [.020, .066] & .030 [.012, .053] & & \\
\midrule
sim.\ upper & .087 & .088 & .155 & .036 \\
\bottomrule
\end{tabular}
\end{table}

\paragraph{Latency.} We timed full chronological decisions with the released implementation on the longest history, without a background, on one CPU thread, as CPU time with three warm-up decisions and the median over the next $20$. A decision recomputes the tf-idf weights from the store, scores the cue and every reference cue, each reference without its own session, standardises and ranks. It takes $24$\,ms at $N=128$ ($K=128$) and $308$\,ms at $N=2048$ ($K=199$), with $95$th percentiles of $25$ and $316$\,ms. Encoding an opening turn with bge-small takes $74$\,ms more.

\paragraph{Reference asymmetry.} Scoring a reference cue without its session takes its best score over $N-m$ memories, where $m$ is the size of that session, and the current cue's over $N$. We simulated independent scores with the session sizes and reference sampling of the chronological replay ($K=\min(199,n)$, $\alpha=0.05$). The false-alarm rate is $5.3\%$ at target $N=16$ against $4.0\%$ with symmetric references, $4.6\%$ against $4.2\%$ at $N=64$, and within $0.2$ points of the symmetric rate from $N=128$ on.

\section{Theory: transfer across questions and simulation}
\label{app:heldout}

\begin{table}[h]
\centering
\small
\caption{Hit rate of the cosine threshold tuned to $5\%$ false alarms on the calibration half, computed from calibration-half score distributions and observed on the evaluation half (LongMemEval-M, explicit questions), with $95\%$ bootstrap intervals over questions. $p_N$ falls below the pool resolution $1/M\approx4.2\times10^{-4}$ from $N=128$ on.}
\label{tab:heldout}
\begin{tabular}{lcccc}
\toprule
 & \multicolumn{2}{c}{bge-small} & \multicolumn{2}{c}{bge-large} \\
$N$ & calibration & observed & calibration & observed \\
\midrule
8 & .68 [.62, .73] & .71 [.65, .76] & .75 [.70, .80] & .78 [.74, .83] \\
32 & .51 [.45, .57] & .56 [.50, .61] & .58 [.52, .64] & .63 [.57, .68] \\
128 & .32 [.27, .38] & .35 [.31, .41] & .34 [.28, .39] & .37 [.31, .43] \\
512 & .14 [.10, .18] & .18 [.14, .22] & .19 [.15, .24] & .21 [.16, .26] \\
2048 & .09 [.06, .13] & .08 [.05, .12] & .11 [.08, .15] & .10 [.07, .13] \\
\bottomrule
\end{tabular}
\end{table}

\paragraph{Across questions.} Table~\ref{tab:heldout} shows the finite-pool hit rate from the calibration half. It is lower than the observation at every $N$ below $2048$, by at most $4.9$ points with bge-small and $6.0$ points with bge-large, and the intervals overlap throughout.

\paragraph{Across store sizes.} A power law $\Hit_N\propto N^{-\kappa}$ fitted to the observed hit rates at $N=32$, $64$ and $128$ gives $\kappa=0.33$ (bge-small) and $0.37$ (bge-large). Extrapolated to $N=256$, $512$, $1024$ and $2048$ it predicts $0.28$, $0.22$, $0.18$ and $0.14$ with bge-small, where $0.25$, $0.18$, $0.14$ and $0.08$ are observed, and $0.29$, $0.22$, $0.17$ and $0.13$ with bge-large, against $0.27$, $0.21$, $0.15$ and $0.10$.

\paragraph{Simulation.} With null scores $X\sim\mathrm{Exp}(1)$ and evidence scores $Y\sim\mathrm{Exp}(\mu)$ of rate $\mu$, $R(p)=p^\mu$ exactly, so $\kappa=\mu$ and $\ell\equiv1$. For $\mu=0.3$ and $0.6$ and $N$ from $8$ to $262{,}144$, the ratio of the exact hit rate to the limit $C(\alpha,\kappa)\lambda^\kappa N^{-\kappa}$ of Theorem~\ref{thm:power} is $1.0005$ at $N=8$ and at most $1.0001$ from $N=64$ on. In the binormal model the local exponent approaches its limit $1/s^2$ slowly: with $d=2$, $s=1$ it is $0.40$ at $N=8$, $0.54$ at $N=512$ and $0.64$ at $N=262{,}144$; with $d=3.6$, $s=1.3$ it rises from $0.09$ to $0.24$ against a limit of $0.59$. The exact hit rate stays within $2\%$ of the upper bound $R(p_N)$ of Proposition~\ref{prop:sandwich} in both cases.

\paragraph{Every gate.} Table~\ref{tab:anygate} compares, in the binormal model with $s=1$ and $\alpha=0.05$, the best-score gate, which by Theorem~\ref{thm:anygate}(ii) has the largest hit rate of any gate in this model, with the likelihood-ratio test $N^{-1}\sum_ie^{dZ_i-d^2/2}$, the most powerful test of a positive against a null store, and with the bound of Theorem~\ref{thm:anygate}(i). The best-score columns are exact, computed from Proposition~\ref{prop:sandwich} and $1-(1-\alpha)^{(N-1)/N}\Phi(\Phi^{-1}((1-\alpha)^{1/N})-d)$. The likelihood-ratio firing rate is estimated from $10^6$ positive stores, with its threshold the empirical $95\%$ quantile on $10^6$ separate null stores; given the simulated threshold, its binomial standard error is at most $0.0005$.

\begin{table}[h]
\centering
\small
\caption{Binormal check of Theorem~\ref{thm:anygate} ($s=1$, $\alpha=0.05$). Firing rates are on positive stores; the bound is capped at $1$.}
\label{tab:anygate}
\begin{tabular}{rrcccc}
\toprule
$d$ & $N$ & best-score hit & best-score firing & likelihood-ratio firing & bound \\
\midrule
1 & 8 & .067 & .109 & .116 & .151 \\
1 & 64 & .015 & .064 & .070 & .086 \\
1 & 512 & .003 & .053 & .056 & .063 \\
1 & 2048 & .001 & .051 & .054 & .056 \\
2 & 8 & .308 & .342 & .346 & .614 \\
2 & 64 & .122 & .167 & .171 & .249 \\
2 & 512 & .042 & .091 & .095 & .121 \\
2 & 2048 & .020 & .069 & .073 & .085 \\
3 & 8 & .690 & .708 & .710 & 1 \\
3 & 64 & .433 & .466 & .469 & 1 \\
3 & 512 & .233 & .274 & .278 & .917 \\
3 & 2048 & .144 & .188 & .191 & .483 \\
\bottomrule
\end{tabular}
\end{table}

\section{Growth of the null maximum}
\label{app:growth}

\begin{table}[h]
\centering
\small
\caption{Mean and $95\%$ quantile of the largest null $z$-score at the smallest and largest store ($N=8$ and $2048$ for LongMemEval-M, $8$ and $256$ for LoCoMo), and the residual sums of squares of the fits $a+b\sqrt{\ln N}$ and $a+b\ln N$ over all store sizes. ``ratio'' is the growth of the lexical quantile divided by that of the semantic one.}
\label{tab:growth}
\begin{tabular}{llccccc}
\toprule
setting & channel & mean & RSS $\sqrt{\ln N}$ / $\ln N$ & $95\%$ quantile & RSS $\sqrt{\ln N}$ / $\ln N$ & ratio \\
\midrule
\multirow{2}{*}{M, explicit, small} & semantic & 1.51 $\to$ 4.14 & .011 / .024 & 2.92 $\to$ 5.68 & .067 / .010 & \multirow{2}{*}{2.6} \\
 & lexical & 0.83 $\to$ 5.39 & .73 / .21 & 2.41 $\to$ 9.67 & 2.35 / .84 & \\
\multirow{2}{*}{M, explicit, large} & semantic & 1.46 $\to$ 4.18 & .010 / .029 & 2.83 $\to$ 5.84 & .025 / .025 & \multirow{2}{*}{2.4} \\
 & lexical & 0.83 $\to$ 5.39 & .73 / .22 & 2.36 $\to$ 9.55 & 2.16 / .79 & \\
\multirow{2}{*}{M, implicit, small} & semantic & 1.43 $\to$ 3.60 & .002 / .033 & 2.63 $\to$ 4.63 & .006 / .018 & \multirow{2}{*}{2.7} \\
 & lexical & 0.96 $\to$ 5.05 & .47 / .11 & 2.42 $\to$ 7.79 & .64 / .15 & \\
\multirow{2}{*}{M, implicit, large} & semantic & 1.37 $\to$ 3.63 & .002 / .031 & 2.56 $\to$ 4.70 & .008 / .035 & \multirow{2}{*}{2.5} \\
 & lexical & 0.96 $\to$ 5.05 & .47 / .11 & 2.41 $\to$ 7.79 & .61 / .14 & \\
\multirow{2}{*}{LoCoMo, small} & semantic & 1.34 $\to$ 2.69 & $<$.001 / .008 & 2.49 $\to$ 3.98 & .002 / .003 & \multirow{2}{*}{2.8} \\
 & lexical & 1.45 $\to$ 4.62 & .091 / .014 & 3.73 $\to$ 7.87 & .057 / .004 & \\
\multirow{2}{*}{LoCoMo, large} & semantic & 1.33 $\to$ 2.75 & $<$.001 / .006 & 2.60 $\to$ 4.23 & .007 / $<$.001 & \multirow{2}{*}{2.5} \\
 & lexical & 1.45 $\to$ 4.62 & .090 / .013 & 3.74 $\to$ 7.87 & .054 / .001 & \\
\bottomrule
\end{tabular}
\end{table}

The mean of the semantic maximum is fitted better by $\sqrt{\ln N}$ and the mean of the lexical maximum by $\ln N$ in every setting. For the $95\%$ quantile the lexical channel again prefers $\ln N$, while the semantic channel prefers $\ln N$ in two settings, $\sqrt{\ln N}$ in three, including both with implicit cues, and neither in one (Table~\ref{tab:growth}).

\section{Several relevant memories}
\label{app:sensitivity}

\begin{table}[h]
\centering
\small
\caption{Hit rate on LongMemEval-M (bge-small, explicit questions, evaluation half) when the positive store holds, besides $N-m$ null turns, one evidence turn (one), every evidence turn of the question (all), one evidence turn three times (repeated), or every user turn of one evidence session (session). For the session stores, ``answer'' counts a hit only on the evidence turn and ``any'' on any turn of the session. Null trials and thresholds are those of the main run.}
\label{tab:sensitivity}
\begin{tabular}{lcccccc}
\toprule
gate & $N$ & one & all & repeated & session, answer & session, any \\
\midrule
\multirow{3}{*}{cosine, oracle per $N$} & 8 & .71 & .80 & .71 & .47 & .81 \\
 & 128 & .35 & .51 & .36 & .28 & .43 \\
 & 2048 & .08 & .12 & .08 & .07 & .12 \\
\midrule
\multirow{3}{*}{$z$-score, oracle per $N$} & 8 & .83 & .89 & .82 & .61 & .89 \\
 & 128 & .54 & .64 & .53 & .45 & .61 \\
 & 2048 & .17 & .23 & .15 & .14 & .19 \\
\midrule
\multirow{3}{*}{extreme-value} & 8 & .77 & .84 & .70 & .59 & .82 \\
 & 128 & .53 & .61 & .36 & .39 & .52 \\
 & 2048 & .37 & .48 & .34 & .30 & .41 \\
\midrule
\multirow{3}{*}{cue-swap} & 8 & .83 & .90 & .84 & .63 & .92 \\
 & 128 & .53 & .64 & .54 & .44 & .61 \\
 & 2048 & .22 & .29 & .22 & .19 & .26 \\
\bottomrule
\end{tabular}
\end{table}

\section{TriggerBench}
\label{app:triggerbench}

We use the positive-clean and negative-clean variants of the released prospective-memory set (commit \texttt{9e2185c}), both lengths. The memories of an item are its user turns before the trigger; turns labelled \texttt{constraint}, \texttt{rule\_constraint} or \texttt{condition\_constraint} are the evidence. Grown stores add filler user turns drawn from the 40K split, excluding turns that appear in the base split, three draws per item and size; the background for standardisation is $300$ further filler turns. A size enters for an item only if the item's full store fits, so the positive and null conditions of a scenario cover the same sizes, and a store of exactly $N$ turns takes no filler; at $N=8$ this leaves $173$ scenarios, against $245$ at the other sizes. Thresholds of the fixed cosine gate ($N=8$) and of the per-size $z$-score gate are the $95\%$ quantiles of the null condition on a hash-defined half of the items; Table~\ref{tab:triggerbench} reports the other half.

\begin{table}[!ht]
\centering
\small
\caption{TriggerBench, bge-small, evaluation half. Positive: hit rate (the gate speaks and its top memory is the constraint). Null: firing rate with the constraint removed. Negative: firing rate on the benchmark's resolved or non-applicable controls. ``base'' and ``40K'' are the natural conversations ($2$--$16$ user turns, about $8$ on average, and up to about $640$); the other columns are grown stores.}
\label{tab:triggerbench}
\begin{tabular}{llccccccc}
\toprule
condition & gate & 8 & base & 32 & 128 & 40K & 512 & 2048 \\
\midrule
\multirow{4}{*}{positive} & $z$-score, oracle & .05 & .09 & .08 & .08 & .08 & .08 & .07 \\
 & cue-swap & .18 & .27 & .24 & .21 & .19 & .18 & .16 \\
 & recall-v1 & .17 & .27 & .26 & .25 & .22 & .23 & .21 \\
 & fixed cosine & .04 & .05 & .05 & .05 & .05 & .05 & .05 \\
\midrule
\multirow{4}{*}{null} & $z$-score, oracle & .04 & .04 & .03 & .04 & .04 & .04 & .05 \\
 & cue-swap & .37 & .36 & .32 & .27 & .20 & .22 & .16 \\
 & recall-v1 & .48 & .50 & .47 & .45 & .34 & .44 & .44 \\
 & fixed cosine & .10 & .10 & .10 & .10 & .11 & .11 & .12 \\
\midrule
\multirow{4}{*}{negative} & $z$-score, oracle & .17 & .16 & .15 & .15 & .15 & .15 & .14 \\
 & cue-swap & .62 & .65 & .53 & .44 & .42 & .37 & .31 \\
 & recall-v1 & .70 & .76 & .67 & .63 & .54 & .60 & .58 \\
 & fixed cosine & .17 & .15 & .15 & .15 & .17 & .16 & .17 \\
\bottomrule
\end{tabular}
\end{table}

The constraint has the highest cosine in the natural base stores for $71\%$ of logical-adherence, $67\%$ of safe-coding, $57\%$ of attention-recovery, $37\%$ of temporal-grounding and $13\%$ of state-tracking items, and in the 40K stores for $61\%$, $29\%$, $57\%$, $1\%$ and $0\%$.

\section{Experimental details}
\label{app:details}

Embeddings use a maximum length of $256$ tokens and no query instruction, matching the deployed plugin. For implicit cues, Qwen2.5-14B received the question and its answer and was asked for one first-person message sent during an everyday task in which recalling the memory would help; the message had to avoid asking about the memory, the phrases ``remember'' and ``did I'', and the answer itself. Rewrites were dropped if they contained the answer string, one of these phrases or ``I told you'', or fewer than ten characters. A manual check of a random sample found most rewrites acceptable but repetitive in form, often beginning with ``I should check'', and a few too vague to evoke any memory. All runs used a single RTX 3090.

\section{Annotation}
\label{app:annotation}

The false alarms were sampled under the version-1 protocol, with the background from the null pool, and each question--memory pair is kept once across sizes, so the sample describes the kind of memory that fires and not its frequency under the main protocol. With the first $38$ items, labelled before the translation fix, the majority-vote helpful shares at $N=64$, $512$ and $2048$ are $5\%$, $5\%$ and $11\%$ of all $240$ false alarms. The guideline defined three labels and a fourth for undecidable items. \emph{Helpful}: when responding to what the user said, the assistant would use concrete content from the memory, including a partial answer such as one amount of several or a date needed for a time span; possibly related background does not count. \emph{Same topic, not useful}: the topic matches but the response would not use the memory. \emph{Unrelated}: little to do with the cue. Each LongMemEval session carries a date, so a memory saying that something happened ``today'' supplies a date. The guideline did not state this, and several disagreements stemmed from it. The human annotator saw the English text and a Chinese machine translation produced jointly for cue and memory with proper nouns kept in English; an earlier translation pass that translated them separately caused errors on the first $38$ items. Those items are excluded from the agreement statistics also because the Claude annotator saw the human labels for them before labelling; the remaining $232$ were labelled by all three without access to any other labels. The LLM annotators were Claude Opus~5.5 and GPT-6 Sol, each given the English text and the guideline in a fresh session. The human marked $5$ items as undecidable, $2$ of them among the first $38$, which leaves $229$ of the $232$ items for the agreement statistics.
\end{document}